\documentclass{l4dc2024}

\usepackage{bm}
\usepackage{enumitem}
\usepackage{amsmath}
\usepackage{amssymb,amsfonts}
\usepackage{algorithm}
\usepackage{algorithmic}
\usepackage{graphicx,color}
\usepackage{mathrsfs}
\usepackage{url}
\usepackage{color}
\usepackage{array}
\usepackage{mathtools}
\usepackage{multicol}
\usepackage{appendix}
\usepackage{tikz}
\usetikzlibrary{shapes,arrows}
\usetikzlibrary{calc}
\usepackage{verbatim}
\usetikzlibrary{backgrounds,datavisualization}
\usetikzlibrary{datavisualization.formats.functions}
\usetikzlibrary{3d}
 \usepackage{pgfplots}
 \usetikzlibrary{positioning,fit}

\usepackage{times} 

\tikzstyle{two dot dash}= [dash pattern=on 3pt off 1pt on \the\pgflinewidth off 1pt on \the\pgflinewidth off 1pt]
\tikzstyle{three dot dash}= [dash pattern=on 3pt off 1pt on \the\pgflinewidth off 1pt on \the\pgflinewidth off 1pt on \the\pgflinewidth off 1pt]

\usepackage[prependcaption,colorinlistoftodos]{todonotes}

\definecolor{LB}{RGB}{65,105,225}
\definecolor{LO}{RGB}{210,127,0}
\definecolor{LP}{cmyk}{0, 0.7808, 0.4429, 0.1412}
\definecolor{LG}{rgb}{0.13, 0.55, 0.13}
\definecolor{LR}{rgb}{0.77, 0.01, 0.2}

\newcommand{\real}{\mathbb{R}}

\newcommand{\A}{\mathcal{A}}

\newcommand{\D}{\mathcal{D}}
\newcommand{\R}{\mathcal{R}}

\newcommand{\argmax}[2] {\mathrm{arg}\max\limits_{#1}#2}

\DeclareSymbolFont{bbold}{U}{bbold}{m}{n}
\DeclareSymbolFontAlphabet{\mathbbold}{bbold}

\newcommand\oprocendsymbol{\hbox{$\square$}}
\newcommand\oprocend{\relax\ifmmode\else\unskip\hfill\fi\oprocendsymbol}

\title[Fusing Multiple Algorithms for Heterogeneous Online Learning]{Fusing Multiple Algorithms for Heterogeneous Online Learning}
\usepackage{times}
\author{
   \Name{Darshan Gadginmath} \Email{dgadg001@ucr.edu}\\
 \Name{Shivanshu Tripathi} \Email{strip008@ucr.edu}\\
 \Name{Fabio Pasqualetti} \Email{fabiopas@engr.ucr.edu}\\
  \addr University of California, Riverside  
{\thanks{This material is based upon work supported in part by awards
ONR-N00014-19-1-2264, ARO W911NF-20-2-0267, and AFOSRFA9550-19-1-0235.}}
}

\begin{document}

\maketitle


\begin{abstract} This study addresses the challenge of online learning in
 contexts where agents accumulate disparate data, face resource constraints,
 and use different local algorithms. This paper introduces the Switched
 Online Learning Algorithm (SOLA), designed to solve the heterogeneous online
 learning problem by amalgamating updates from diverse agents through a dynamic
 switching mechanism contingent upon their respective performance and available
 resources. We theoretically analyze the design of the selecting mechanism to
 ensure that the regret of SOLA is bounded. Our findings show that the number
 of changes in selection needs to be bounded by a parameter dependent on the
 performance of the different local algorithms. Additionally, two test cases
 are presented to emphasize the effectiveness of SOLA, first on an online
 linear regression problem and then on an online classification problem
 with the MNIST dataset. 
\end{abstract}



\section{Introduction}\label{introduction}

Multi-agent learning frequently involves scenarios in which agents gather
disparate data at varying rates, collectively seeking to address an online
optimization problem. Some instances include collaborative localization,
search-and-rescue operations, coverage control, etc. Compounding the complexity
of these scenarios is the constraint of limited data processing and
computational capabilities, and the need for time-sensitive decision-making.
Conventionally, either heterogeneous data or gradients are pooled~\cite
{HZ-etal:2018,YE-etal:2021, SAA-SK-KR:2023}, or heterogeneity is ignored and a
distributed learning algorithm is deployed~\cite
{BLB-etal:2023,HZ-etal:2021}. The former approach raises privacy concerns,
while the latter proves suboptimal due to the inherent heterogeneity of the
data. Effectively addressing the heterogeneity of data and resource constraints
in an online learning problem remains an unresolved challenge.  

An alternative approach to mitigate the challenges posed by data heterogeneity
and resource constraints is to employ distinct algorithms tailored to the
specific characteristics of the data, computation, and communication resources.
Nevertheless, adopting distinct algorithms introduces the risk of
underutilizing the available data. In this study, we present a systematic
method to integrate updates provided by distinct algorithms to solve the
heterogeneous online learning problem.

\subsection{Problem Statement}

We seek to solve the following online minimization problem:
\begin{align}
\min_x F(x, \bigcup\limits_i^M \D^i(t)). \label{prob:online-learning-individual}
\end{align}

Here, $x\in \real^n$ is a parameter that needs to be collectively estimated by
agents $i \in \{1,2,\dots, M\}$. The data gathered by each agent up to time $t$
is denoted by $\D^i(t)$. More precisely, let $T^i \in \real$ denote the set of time instances when agent $i$ collects new data. Then $T^i = \{t^i_1,t^i_2,\dots\}$, where $t^i_j$ denotes the $j^{th}$ round of sampling by agent $i$. We use the set $T$ to denote all the time instances when new data is available. That is, $T = \bigcup_i~T^i =\{t_1, t_2, \dots\}$.
We drop the superscript $i$ to denote a time instance when new data is acquired by any agent. The elements of $T$ satisfy $t_1 \leq t_2 \leq \dots \bar{t}$, where $\bar{t} = \max_{i,j} t^i_j$. Let the samples collected by agent $i$ at the $j^{th}$ round be denoted by $s(t^i_j)$. The data collected by agent $i$ until time $t$ is 
\begin{align}
\D^i(t) = \bigcup\limits_{\substack{t^i_j \in T^i,t^i_j \leq t}} s(t^i_j).
\end{align} 
Every agent $i$ is constrained by their computation power and communication
capabilities. Hence, they need to employ a local algorithm $\A^i$ with their
local data to solve the online optimization problem. For instance, a single
agent can employ a centralized algorithm such as gradient descent
(GD), stochastic gradient descent (SGD), or batch-wise GD~\cite{RD-etal:2019}. An agent who is composed of a system of smaller units can
collectively employ distributed algorithms such as Decentralized SGD or
Federated Learning depending on the availability of resources. As an
example, consider the following situation: an e-commerce entity equipped with
data centers strategically situated across varied geographical locations,
adopts an asynchronous data acquisition methodology from online users for the
purpose of targeted advertisement display. In this scenario, each center
employs a localized algorithm to process its specific dataset, owing to
limitations in both computational and communicative capabilities. 

Agents may choose to delay their decision-making process in order to accumulate
an ample amount of data or computational resources necessary for solving
problem~\eqref{prob:online-learning-individual}. However, this approach may be
suboptimal, given that decisions are frequently time-sensitive. A practical
example is that of naval vessels mapping the sea for adversarial entities.
Vessels positioned at varying distances from the shoreline collect sonar data,
yet their computational and communicative capabilities are constrained by their
respective locations. Specifically, vessels in closer proximity to the coast
benefit from superior computational resources, albeit with diminished data
quality, as elucidated by~\cite{CMF-FBJ:2002}. In this scenario, waiting to
gather sufficient resources can be extremely dangerous. Therefore, the updates
from different algorithms need to be fused in an online fashion as soon as
updates are available to solve the problem~\eqref
{prob:online-learning-individual}.

\subsection{Related Work} Distributed online optimization has been extensively
 studied as detailed in~\cite
 {HBM:2017,SCHH-etal:2021,XL-LX-NL:2023}. However, these studies do not
 consider cooperatively using different algorithms in a constrained
 time-sensitive setting. Popular algorithms such as decentralized SGD and
 Federated Learning in the presence of asynchronous agents were studied
 in~\cite{JJ-etal:2021,YC-etal:2020}. In~\cite{JJ-etal:2021}, decentralized
 SGD is proposed with asynchronous agents but it does not incorporate agents running
 different algorithms and it also requires extensive communication between
 agents. Asynchronous online Federated Learning \cite{YC-etal:2020} requires
 the presence of a coordinator, and still does not fuse different algorithms.
 Model fusion has received attention in supervised learning~\cite
 {WL-etal:2023}. Model fusion in online learning case has received
 significantly less interest due to its complexities~\cite
 {DJF-AR-KS:2015,TNH-etal:2019,AC:2019}. These works typically consider selecting models from several algorithms at every time step. Particularly,~\cite{AC:2019} shows
 that algorithms with bounded regrets can be fused simply by averaging the
 parameters and still maintain bounded regret. However, in our case, we seek to
 fuse the updates from different agents by employing only a single agent at any
 given time. Therefore, it is unclear how agents with different data and
 resources could cooperatively solve problem~\eqref
 {prob:online-learning-individual} by running their own local algorithm.

\subsection{Contributions}
We provide an algorithm called the Switched Online Learning Algorithm (SOLA) to
   solve Problem~\eqref{prob:online-learning-individual} by fusing the
   updates from agents running different algorithms. We solve the considered
   problem by switching between the agents and fusing their updates based on
   their performance. We provide a sufficient condition to guarantee a bound on the regret of
   SOLA based on the rate at which different algorithms are chosen. To this
   end, we model SOLA as a switched dynamical system and ensure its
   contractivity. We numerically analyze the performance of our algorithm for the online
   linear regression problem and also the online classification problem with
   the MNIST dataset\footnote{Code repository: \url{https://github.com/Shivanshu-007/Heterogeneous-online-optimization}
}. 

\section{Switched Online Learning Algorithm (SOLA)}

In this section, we describe the proposed Switched Online Learning Algorithm
(SOLA). The input to any local algorithm $\A^i$ at time $t$ is the data $\D^i
(t)\in \real^{m(t)\times p}$, and the parameter $x \in \real^n$. Note that
the dimension of $\D^i(t)$ is dependent on time as each agent acquires new data
over time. The number of samples is given by $m(t)$ and the number of features
is $p$. Henceforth, we simply use the variable $k$ to denote the discrete instances $t_k \in T$, 
i.e. $k \in \{1,2,\dots, |T|\}$. The update provided by algorithm $\A^i$ is given by the map
$\A^i(x, \D^i(k)): \real^{n} \times \real^{m(k)\times p} \rightarrow \real^
{n}$. At any instance $k$, the \emph{selecting signal} $\sigma(k): \{1,2,...,|T|\} \rightarrow \{1,2,\dots,M\}$ selects an agent $i$ if the agent has new data. Agent $\sigma(k)$ uses its local algorithm
$\A^{\sigma(k)}$ to update the parameter $x(k-1)$. The update provided by
$\A^{\sigma(k)}$ is used to solve the problem~\eqref
{prob:online-learning-individual} as
\begin{align}\label{eqn:sola-dynamics}
x(k) = \alpha(k) \ \A^{\sigma(k)}(x(k-1),\D^{\sigma(k)}(k)) + (1-\alpha(k)) \ x(k-1).
\end{align}
We call $\alpha(k) \in [0,1]$ as the \emph{fusing variable}. We introduce
it to smoothly incorporate new updates from the chosen local algorithms $\A^
{\sigma(k)}$. The fusing variable depends on the performance of the algorithm
of the local algorithms. Let us define a performance metric as $P(x,\D): \real^
{n} \times \real^{m \times p} \rightarrow \real_{\geq0}$. Common performance
metrics for classification problems are precision, recall, true positive rate, etc.~\cite{DJH:2012}. In regression problems, performance is often computed as the inverse of the norm of error, or the trace of the inverse of error covariance~\cite
{SMK:1993}. We are particularly interested in metrics that have higher values to signify better performance. Given the performance metric, the fusing variable $\alpha(k)$ is
defined as:
\begin{align}
\alpha(k) &= \frac{P(x(k^+),\D^{\sigma(k)}(k))}{P(x(k^+),\D^{\sigma(k)}(k)) + P(x(k-1),\D^{\sigma(k-1)}(k-1))},\label{eqn:fusing-var}
\end{align}
where $x(k^+) = \A^{\sigma(k)}(x(k-1),D^{\sigma(k)}(k))$. Note that $\alpha(k)$ is
using the performance of $\A^{\sigma(k)}$ to incorporate its update. If the
performance is poor on the local data $\D^{\sigma(k)}(k)$, the fusing variable $\alpha(k)$ is closer to
$0$, whereas $1-\alpha(k)$ is closer to 1. This implies that the former parameter $x(k-1)$ has a
higher influence on the updated parameter $x(k)$. Similarly, if the performance of $\A^{\sigma(k)}$ is
superior to that of $x(k)$, more weight is given to $x(k^+)$. We describe the
proposed algorithm SOLA in detail in Algorithm~\ref{algo:SOLA}.
\begin{remark}{\textbf{(Naive switching):}} The design of the fusing variable
 $\alpha(k)$ is crucial to SOLA because the updates from $\A^{\sigma(k)}$ can
 be vastly different from $x(k)$. For instance, by setting $\alpha(k) = 1$, the
 update of one agent acts as the input to the subsequent agents. This is a
 naive method of switching between agents without any consideration for the
 performance of each agent. Switching naively
 between different agents may cause abrupt and large changes in the parameters,
 which may be undesirable. This is illustrated in the online regression problem
 represented in Figure~\ref{fig:naive-linear-regression}. SOLA chooses between two agents running centralized GD and decentralized SGD with five sub-units. Figure~\ref{fig:naive-linear-regression} compares the naive case when $\alpha(k) = 1$, and the case when $\alpha(k)$ is dependent on the performance of the local algorithms. We see frequent jumps in
 the parameter $x(k)$ resulting in a chattering behavior with an improper
 choice of $\alpha(k)$. However, when $0\leq\alpha(k)\leq 1$, we see an improvement in the performance and
 the chattering behavior is absent. The detailed simulation setting is provided
 in Section \ref{sim:linear reg}. 
\end{remark}
\begin{figure*}
\centering
\begin{tikzpicture}
  \node (img1)  {\includegraphics[width=0.450\textwidth]{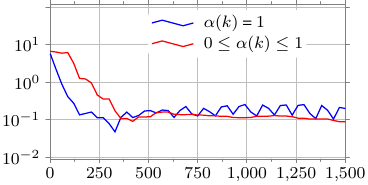}};
  \node[above of= img1, node distance=0cm, yshift=-2.2cm,font=\color{black}]{time ($t$)};  
  \node[left of= img1, node distance=0cm, rotate=90, anchor=center,yshift=3.9cm,font=\color{black}] {$\|x(t) - x^*\|$};
\end{tikzpicture}
\vspace*{-1.2em}
\caption{Online linear regression with SOLA comparing a naive choice of the
 fusing variable $\alpha_t=1$ and a fusing variable that incorporates the
 performance of each agent. It is evident that a naive choice of the fusing
 variable results in frequent jumps of the parameter $x(t)$. }
\label{fig:naive-linear-regression}
\end{figure*}
\begin{remark}{\textbf{(Choice of performance metric):}} SOLA uses updates from
 different agents for a general online optimization problem. Therefore, the
 optimal choice of a performance metric is dependent on the cost function, data
 and the constraints of each agent. For instance, consider the online linear
 regression problem. The recursive least squares algorithm in~\cite
 {SMK:1993} incorporates data arriving incrementally to solve an online linear
 regression problem. Although the algorithm operates with only a single agent,
 it incorporates incoming updates by weighting them with the inverse of the
 error covariance. It can be shown that the choice of using the inverse of the
 error covariance in the recursive least-squares algorithm produces the Best
 Linear Unbiased Estimator (BLUE).  
\end{remark}

\begin{algorithm}[tbh] 
\caption{Switched Learning Algorithm}
\begin{algorithmic}\label{algo:SOLA}
\STATE\textbf{Initialize:} $x(0),~\sigma(0),~\D^{\sigma(0)}(0)$.\\
\SetAlgoLined
\STATE 1. \textbf{For} $k = 1 $ to $|T|$\textbf{:}\\
\STATE 2. \quad \quad Select agent $i$ using switching signal, $\sigma(k) = i$. \\ 
\STATE 3. \quad \quad Agent $\sigma(k)$ receives $x(k-1)$ and $P(x(k-1),\D^\sigma(k-1)(k-1))$ from $\sigma(k-1)$.  \\
\STATE 4. \quad \quad Update parameter $x(k-1)$ using algorithm $\A^{\sigma(k)}$: $x(k^+) = \A^{\sigma(k)}(x(k-1),D_{\sigma(k)}(k))$. \\
\STATE 5. \quad \quad Compute performance metric $P(x(k^+)),D_{\sigma(k)}(k))$ and fusing variable $\alpha(k)$~\eqref{eqn:fusing-var}.\\
\STATE 6. \quad \quad Incorporate update: $x(k) = \alpha(k) \A^{\sigma(k)}(x(k-1),\D^{\sigma(k)}(k)) + (1-\alpha(k)) x(k-1)$.\\
\end{algorithmic}
\end{algorithm}
 
\section{Design of Selecting Signal}

The selecting signal $\sigma(k)$ determines the choice of subsystem for the
update, and in turn, determines the regret of SOLA. In this section, we analyze
the effect of the switching signal on the regret of the Switched Learning
Algorithm. Let us denote any algorithm employed to solve problem~\eqref
{prob:online-learning-individual} as $\mathcal{A}$. The regret of algorithm
$\A$ is denoted by $\R_{\A}$ and is defined for $K\geq1$ as
\begin{align}
\R_\A(K) &= \sum_{k=1}^K F(x(k), \D(k)) - \sum_{k=1}^T F(x^*, \D(k) ).  
\end{align} 
Here $x(k)$ refers to the parameters provided by algorithm $\A$ at time $k$
using data $\D(k)$, and $x^*$ is the optimal parameter given all the data a priori. It is essential for
an algorithm to have bounded regret as it compares the performance of the
algorithm to an optimal choice. We now show that SOLA has bounded regret under
the right conditions of the selecting signal $\sigma(k)$.

We analyze the regret achieved by SOLA by viewing the algorithm as a dynamical
system. Optimization algorithms have received extensive attention from the
view of dynamical systems~\cite
{AE-TE-MP-PLK:2005,VS-OS:2009,SR-JAB:2011,AS-PJ-AT:2012,EDS:2022,PCV-FB:2022,LK-PW-JJS:2022}.
Particularly, algorithms that are stable have been shown to have bounded
regret~\cite{SR-JAB:2011,AS-PJ-AT:2012}. A useful tool to study these
optimization algorithms and stability is contraction theory~\cite
{PCV-FB:2022,LK-PW-JJS:2022}. A contracting algorithm is defined as follows.  
\begin{definition}{(\textbf{Contracting algorithm},~\cite{WL-JJES:1998})}\label{def:contraction}
Let the updates provided by an algorithm $\A$ be given by the dynamical system 
\begin{align}
x(k) = \A(x(k-1),\D(k))\label{eqn:algo-dynamicalsys}
\end{align}
where $x(k)$ and $\D(k)$ are the parameter and data at time $k$, respectively. The differential dynamics of the algorithm is then given by 
\begin{align*}
\delta_{x(k)} = \frac{\partial\A(x(k-1),\D(k))}{\partial x(k-1)} \delta_{x(k-1)}.
\end{align*}
The associated distance for the differential dynamics is denoted by
\begin{align*}
V(x(k))&= \delta_{x(k)}^\top M(x(k)) \delta_{x(k)}, \\
M(x(k)) &= \left(\frac{\partial\A(x(k),\D(t))}{\partial x(k)} \right)^\top  \frac{\partial\A(x(k),\D(t))}{\partial x(k)}.
\end{align*}
Here $M(x(k))$ is a symmetric positive-definite matrix function, and it is uniformly bounded as $\gamma_1 I \leq M(x(k)) \leq \gamma_2 I$, for all $x(k)$.
The algorithm $\A$ is said to be contracting if 
\begin{align}
V(x(k)) \leq \beta V(x(k-1)),
\end{align}
where $0 < \beta \leq 1$, is the rate of contraction.
\end{definition}
We show that any contracting algorithm achieves bounded regret in Appendix~\ref
{app:online-stab-contraction}. Given that contracting algorithms achieve
bounded regret, it is sufficient to ensure that SOLA is a contracting algorithm
by designing the switching signal $\sigma(k)$. We make the following
assumptions for our regret analysis.
\begin{enumerate}[label=\textbf{(A\arabic*)}]
  \item \label{assu:convex}$F(x,\D)$ is $\ell$-convex in $x$: $\nabla^2 F(x,\D) > \ell I$. 
  \item \label{assu:cont-metric}Every local algorithm $\A^i$ is contracting with a rate $\beta^i$. Further, for each pair of algorithms $(i,j)$, there exists $\mu^{ij}>1$ such that the distance of differential dynamics is bounded as $V^{i}(x(k)) = \mu^{ij} V^{j}(x(k))$. 
\end{enumerate}
Assumption \ref{assu:convex} is a commonly used in regret analysis for online problems~\cite{HBM:2017}. However, \ref{assu:cont-metric} needs more careful attention as it comes from the perspective of dynamical systems. Several algorithms such as gradient descent~\cite{PCV-FB:2022}, SGD and decentralized SGD algorithms have been proven to be contractive~\cite{NMB-JJES:2020}. When these algorithms are used to solve a common online learning problem, assumption~\ref{assu:cont-metric} captures the relationship between the differential dynamics of the local algorithms. Particularly, it provides an upper bound on the relative difference between the performance of the various local algorithms. We introduce $\bar{\mu}$ and $\bar{\beta}$ which will be useful for our regret analysis.
  \begin{align*}
  \bar{\mu} = \argmax{ij}{\mu^{ij}}, \quad 
  \bar{\beta} = \argmax{i}{\beta^i}
  \end{align*}
The variables $\bar{\mu}$ corresponds to the largest difference in differential dynamics and $\bar{\beta}$ corresponds to the slowest contracting rate. The following theorem addresses the design of the selecting signal such that SOLA achieves bounded regret.
\begin{theorem}{\textbf{(SOLA achieves bounded regret)}} For a selecting signal $\sigma(k)$, let the number of switches between different agents be $N(k_1,k_2)$ over any horizon $\left[k_1,k_2\right]$, where $k_2>k_1$. That is,
\begin{align}
N(k_1,k_2) &= \sum\limits_{k=k_1+1}^{k_2} r(k),\quad \mathrm{where~} 
r(k) = \begin{cases}
0, \quad \sigma(k) = \sigma(k-1),\\
1 ,  \quad \mathrm{otherwise}.
\end{cases}
\end{align}
Then, if the number of switches satisfies
\begin{align}
N(k_1,k_2) &\leq N_0 + \frac{k_2 - k_1}{\tau},
\end{align}  
where $N_0>0$ is a constant, and $\tau= -\frac{\ln(\bar{\mu})}{\ln (\bar{\beta})}$, 
then SOLA achieves bounded regret: $\R(K) \leq \epsilon(K)$, where $\epsilon(K)$ is a decreasing function in $K$. 
\label{thm:selecting-signal}
\end{theorem}
We refer the reader to Appendix~\ref{app:selecting-signal} for the proof of Theorem~\ref{thm:selecting-signal}. In the literature related to switched systems~\cite{DL:2003}, $N_0$ is typically referred to as the chatter bound and $\tau$ is referred to as the average dwell time. 
\begin{remark}{\textbf{(Effect of cost function, data, and algorithm on the number of switches)}}
It is important to note that $\tau$ captures the effect of the cost function, the data and the characteristics of the different local algorithms on the admissible switching signal in Theorem~\ref{thm:selecting-signal}. For instance, poor data or a slow learning rate for the agent $i$ results in high values of $\beta^i$. This leads to a 
higher value of $\bar{\beta}$, which in turn means that the number of switches $N(k_1,k_2)$ 
needs to be small. Conversely, algorithms that learn fast allow for fast switching. Another 
aspect of $\tau$ is the similarity of the behavior of local algorithms captured through Assumption~\ref{assu:cont-metric}. If the different local algorithms behave similarly, $\bar{\mu}\approx 1$. This leads to a small value of $\tau$. Conversely, if the algorithms behave very differently $\bar{\mu}$ is larger which restricts the frequency of switching local algorithms.  
\end{remark}

\section{Numerical results}
In this section, we show the effectiveness of SOLA for online linear regression and online classification using the MNIST dataset. 

\subsection{Online linear regression}\label{sim:linear reg}
We conduct experiments for online linear regression with a synthetic dataset. The agents acquire data $\big(A^i(k), B^i(k)\big)$, where $A^i(k) \in \real^{1\times m(k)}$, $B^i(k) \in \real^{3 \times m(k)}$, and $x \in \real^{3\times1}$. The online linear regression problem is 
\begin{align*}
\min\limits_{x} \left|\left| A^{\sigma(k)}(k) - x^\top~B^{\sigma(k)}(k) \right|\right|^2_2.  
\end{align*} 
In this experiment, the number of agents $M=2$. One agent acquires data sampled as $A^1(k) = {x^*}^\top B^1(k) + \zeta^1(k)$, where $\zeta^1(k) \sim \mathcal{N}(0,3I)$.  The other agent collects data $A^2(k) = {x^*}^\top B^2(k) + \zeta^2(k)$, where $\zeta^2(k) \sim \mathcal{N}(0,30I)$. Here, $B^1(k),B^2(k) \sim \mathcal{N}(0,0.5I). $ Figure~\ref{fig:linear-regression} compares the performance of SOLA for different settings. The selecting signal $\sigma(k)$ periodically chooses between the two agents every ten instances, i.e. $N(k,k+10) = 1$. In Figure~\ref{fig:linear-regression}, we compare the performance of SOLA in two different settings, (i) Agent 1 has one sub-unit that uses centralized gradient descent, and agent 2 uses decentralized SGD with 5 sub-units, (ii) Agent 1 has 5 sub-units that perform FedAvg~\cite{BM-etal:2017} and agent 2 uses decentralized SGD. Lastly, we compare the performance of only agent 2 using decentralized SGD without SOLA. From Figure~\ref{fig:linear-regression}, it is evident SOLA with GD and decentralized SGD performs the best, whereas SOLA with FedAvg with decentralized SGD converges slower. Decentralized SGD by itself converges slower and has a higher error. 

In Figure~\ref{fig:linear-regression}(b), we compare SOLA with $M=3$. The three agents perform centralized GD, decentralized GD with five sub-units and FedAvg with five sub-units. The data used by the agents is the same as mentioned above. It is evident that a naive choice of the fusing variable not only causes more error but also leads to chattering. SOLA has a higher error when $M=3$ as compared to when $M=3$. This is because fusing more agents requires a good choice of the fusing variable and data of good quality.

\begin{figure*}
\centering
\begin{multicols}{2}
\begin{tikzpicture}
  \node (img1)  {\includegraphics[width=0.45\textwidth]{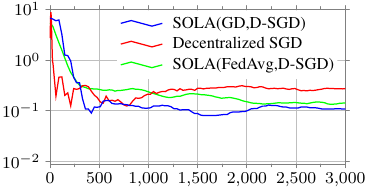}};
  \node[above of= img1, node distance=0cm, yshift=-2.0cm,font=\color{black}]{time}; 
  \node[above of= img1, node distance=0cm, yshift=-2.6cm,font=\color{black}]{(a)};  
  \node[left of= img1, node distance=0cm, rotate=90, anchor=center,yshift=4.2cm,font=\color{black}] {$\|x(t) - x^*\|$};
\end{tikzpicture}\columnbreak
\begin{tikzpicture}
  \node (img1)  {\includegraphics[width=0.45\textwidth]{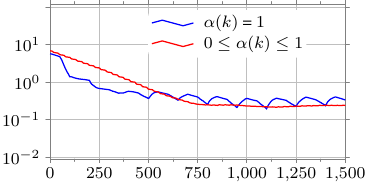}};
  \node[above of= img1, node distance=0cm, yshift=-2.0cm,font=\color{black}]{time};
  \node[above of= img1, node distance=0cm, yshift=-2.6cm,font=\color{black}]{(b)};  
\end{tikzpicture}
\end{multicols}
\vspace*{-2em}
\caption{Online linear regression with SOLA. (a) The blue curve represents SOLA choosing between agents using centralized GD and decentralized SGD, the green curve represents SOLA choosing between agents performing FedAvg and Decentralized SGD, whereas the red curve represents only decentralized SGD. SOLA in both cases performs better than pure decentralized SGD. (b) SOLA with 3 agents performing gradient descent, decentralized SGD and FedAvg compared with SOLA with a naive choice of the fusing variable, $\alpha =1 $. SOLA with good choice of $\alpha$ not only eliminate the chattering behavior, it also has a better overall error. }
\label{fig:linear-regression}
\end{figure*}
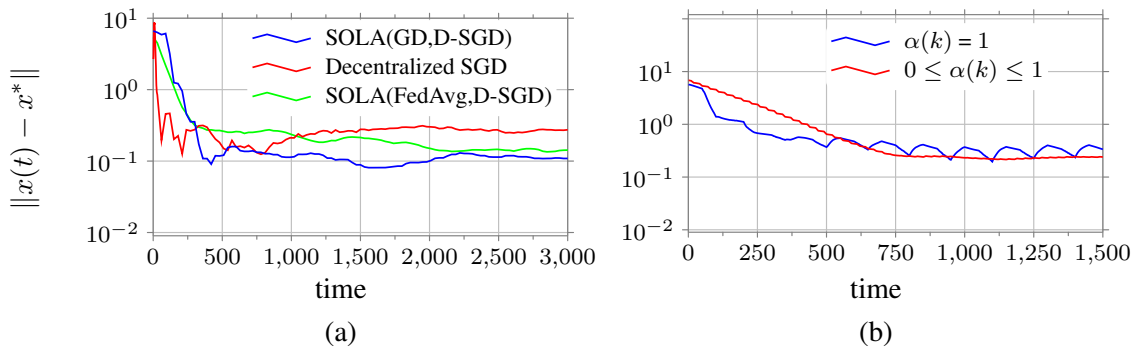

\subsection{Online classification}
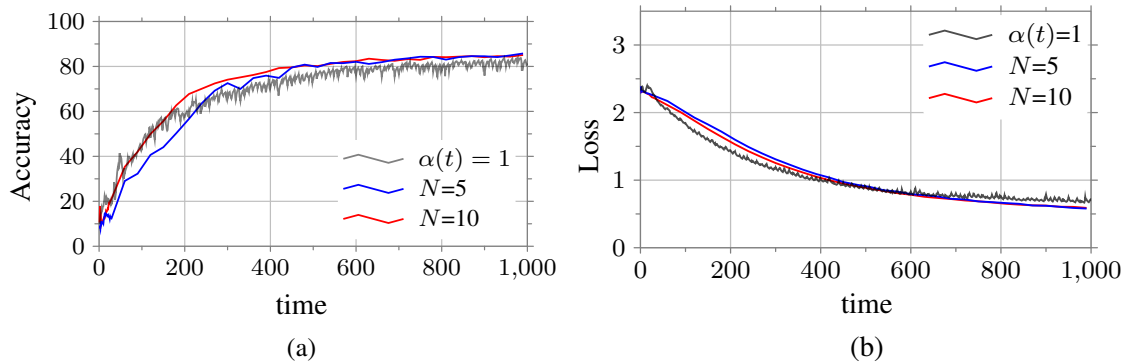
\begin{figure*}[tbh]
\begin{center}

\begin{multicols}{2}
\begin{tikzpicture}
  \node (img1)  {\includegraphics[width=0.450\textwidth]{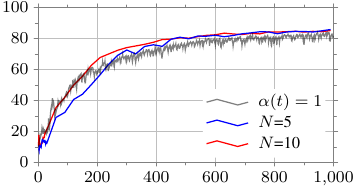}};
  \node[above of= img1, node distance=0cm, yshift=-2.1cm,font=\color{black}]  {time};  
  \node[above of= img1, node distance=0cm, yshift=-2.7cm,font=\color{black}]  {\small (a)};
  \node[left of= img1, node distance=0cm, rotate=90, anchor=center,yshift=3.7cm,font=\color{black}] { Accuracy};
\end{tikzpicture}\columnbreak
\begin{tikzpicture}
  \node (img1)  {\includegraphics[width=0.450\textwidth]{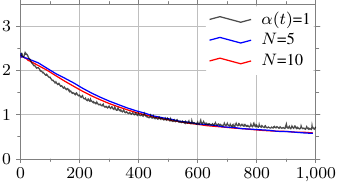}};
  \node[above of= img1, node distance=0cm, yshift=-2.1cm,font=\color{black}]  {time}; 
  \node[above of= img1, node distance=0cm, yshift=-2.7cm,font=\color{black}]  {(b)};
  \node[left of= img1, node distance=0cm, rotate=90, anchor=center,yshift=3.7cm,font=\color{black}] {Loss};
\end{tikzpicture}
\end{multicols}

\end{center}
\vspace*{-2em}
\caption{The testing accuracy and loss for SOLA on the MNIST data set. Here the selecting signal chooses between an agent running decentralized SGD and centralized GD, respectively using the MNIST dataset. (a) accuracy of SOLA, (b) entropy loss for SOLA. } \label{fig:mnist-naive}
\end{figure*}


In this experiment, we consider the online classification problem with $M=2$ and the cost is the cross-entropy loss. Here, one agent samples data from the MNIST dataset and employs centralized gradient descent. It uses a neural network that has one hidden layer with 128 neurons. 
The second agent performs decentralized SGD with sub-units having the same neural network architecture as agent 1. We compare the performance of SOLA for different numbers of sub-units for the second agent. In the first case, there are five sub-units and each sub-unit acquires only two labels. For example, the first sub-unit receives labels $0$ and $1$, the second sub-unit receives labels $2$ and $3$, and so on. In the second case, there are ten sub-units and each sub-unit receives only one distinct label. Further, every sub-unit is restricted to possess only $128$ images (reducing computation). The images sampled by the sub-units are from the MNIST dataset but marred with Gaussian noise of variance $0.5$ and zero mean. Figures~\ref{fig:mnist-naive}(a) and (b) show the accuracy and entropy loss of SOLA on testing data over time. We denote the number of sub-units as $N$ in Figure~\ref{fig:mnist-naive}. We observed that the performance metric $P(x,\D) = \frac{1}{F(x,\D)}$ gave the best performance for both accuracy and loss. The selecting signal periodically chooses between agents every five instances, i.e. $N(k,k+5) = 1$. It can be seen that SOLA still achieves an accuracy close to eighty-two percent for both cases of five and ten sub-units. However, the naive fusing choice with five sub-units has lower accuracy and more chattering.

\section{Conclusions and Future Work} In this work, we considered the scenario
 where agents with different data and resources use different local algorithms
 to solve an online learning problem. Our proposed algorithm, SOLA, provides a
 way to systematically fuse the updates from different algorithms and ensure
 that regret is bounded. We also numerically analyze the performance of SOLA
 for different online learning scenarios. Future directions include the case of
 dynamically changing data distributions, tighter regret bounds and the adversarial 
 case where byzantine agents provide malicious updates.


\begin{appendix}
\section{Contracting Optimizers Achieve Bounded Regret}\label{app:online-stab-contraction}
The connection between the stability of learning algorithms and bounded regret for online learning problems has been studied in ~\cite{TP-SV-LR:2011,SR-JAB:2011,AS-PJ-AT:2012}. Particularly, in~\cite{TP-SV-LR:2011}, the notion of online stability is defined as follows:
\begin{definition}{(\textbf{Online stability})}\label{def:online-stab}
An algorithm $\A$ is said to be online stable if 
\begin{align}
\mathbb{E}\left|\left| F(x^{\A}(k),\D(k)) - F(x^{\A}(k-1),\D(k-1))\right|\right| \leq \epsilon_{os}(k), \quad \forall t,
\end{align}
where $\epsilon_{os}(k)\rightarrow 0$ as $k\rightarrow \infty$.
\end{definition}
The notion of online stability captures the fact that for an online stable
algorithm, the change in cost incurred between any time instances is bounded by
a non-increasing function $\epsilon_{os}(k)$. Importantly, Theorem 18 in \cite{SR-JAB:2011} shows
that online-stable algorithms achieve bounded regret where $\epsilon(K) \leq \sum_{k=1}^K\epsilon_{os}(k)$. 
Also,~\cite{TP-SV-LR:2011} 
shows that iterative gradient-based methods such as GD, and SGD achieve online stability. 
We first model any iterative gradient descent algorithm $\A$ as a perturbated gradient descent. 
\begin{align}
x(t) &= \A(x(t-1),\D) = x(t-1) -\eta~(\nabla F(x(t-1),\D) + \xi(k)). \label{eqn:pertubedGD}
\end{align}
The perturbation to the true gradient is $\xi(k) \sim \mathcal{N}(0,\Sigma(\nabla F(x(k-1)))$ and $\eta$ is a constant learning rate. Further, the covariance of the perturbation $\Sigma(\nabla F(x(k)))$ satisfies $\lim\limits_{t\rightarrow \infty} \Sigma(\nabla F(x(k))) = \bm{0}$,
where $\bm{0}$ denotes the zero matrix. This model is commonly used in the analysis of algorithms such as SGD~\cite{SJ-etal:2017,SM-MDH-DMB:2017,XL-FO:2019} and decentralized SGD~\cite{NMB-JJES:2020}. The dynamics of any single sub-unit of decentralized SGD or FedAvg can be expressed with~\eqref{eqn:pertubedGD}. Further, the average parameter of all the sub-units also follows the same dynamics, however, the noise characteristics of $\xi$ differ. Here, we show that perturbed iterative gradient-based algorithms are contracting as given by Definition~\ref{def:contraction}, and are 
are online stable. 
\begin{theorem}{\textbf{(Gradient-based algorithms are contractive and achieve online stablility)}}
Consider any iterative gradient-based algorithm $\A$ which updates the parameter $x$ as given by~\eqref{eqn:pertubedGD}.
If an algorithm $\A$ is contracting by
Definition~\ref{def:contraction} in expectation, i.e., 
\begin{align*}
\mathbb{E}_{\xi}\left[V(x(k))\right] \leq \beta~\mathbb{E}_{\xi}\left[V(x(k-1))\right],
\end{align*}
then $\A$ achieves online stability. 
\end{theorem}
\begin{proof}
The differential dynamics of the system~\eqref{eqn:pertubedGD} is given by 
\begin{align}
\delta_{x(k)} &= P(x(k-1)) \delta_{x(k-1)} - \eta \delta_{\xi(k-1)}, \\
P(x(k)) &= \frac{\partial\big(x(k) - \eta~\nabla F(x(k))\big)}{\partial x(k)} = I - \eta \nabla^2 F(x(k))
\end{align}
Here $\xi$ is treated as an external signal. For the Euclidean metric $M(x(k)) = I$, the distance of the differential dynamics is given by
\begin{align}
V(x(k)) = \delta_{x(k-1)} P(x(k-1))^2 \delta_{x(k-1)} - 2 \eta \delta_{\xi(k-1)}^\top P(x(k-1)) \delta_{x(k-1)} + \eta^2 ||\delta_{\xi(k-1)}||^2.
\end{align}
Note that $P(x(k)) < I  - \eta \ell I$ due to the $\ell$-convexity of $F$. Therefore, 
\begin{align}
V(x(k)) < (1 - \eta \ell)^2 V(x(k-1))  - 2 \eta \delta_{\xi(k-1)}^\top P(x(k-1)) \delta_{x(k-1)} + \eta^2 ||\delta_{\xi(k-1)}||^2.
\end{align}
Taking the expectation with respect to the perturbation $\xi(t)$, 
\begin{align}
\mathbb{E}[V(x(k))] &< (1 - \eta \ell)^2 \mathbb{E}[V(x(k-1))] + \eta^2 \mathrm{Var}[\xi(k-1)]. 
\end{align}
If the learning rate $\eta$ ensures that
\begin{align}
(1 - \eta \ell)^2 \mathbb{E}[V(x(k))] + \eta^2 \mathrm{Var}[\xi(k-1)] \leq \mathbb{E}[V(x(k))],
\end{align}
then $\A$ is contractive. 
The proof of online stability follows along the same lines as Theorem 2 of \cite{TP-SV-LR:2011}. The cost $F(x(t+1))$ can be expressed using the first two derivatives by Taylor expansion. The stability of the perturbed gradient can be then used to bound the change in cost as given by Definition~\ref{def:online-stab}.

\end{proof}

\section{Proof of Theorem~\ref{thm:selecting-signal}} 
First, consider SOLA with the fusing variable $\alpha(k) = 1$. Then, 
\begin{align}
x(k) &= \A^{\sigma(k-1)}(x(k-1),\D(k-1)).
\end{align}
By assumption~\ref{assu:cont-metric}, we have that between any two instances $k_1$ and $k_2$
\begin{align*}
V^{\sigma(k_2)}(x(k_2)) &\leq \bar{\beta} V^{\sigma(k_2)}(x(k_2-1)) \leq \bar{\mu} \bar{\beta} V^{\sigma(k_2 - 1)}(x(k_2-1)) \\
&\leq {\bar{\mu}}^{N(k_1,k_2)} {\bar{\beta}}^{(k_2 - k_1)} V^{\sigma(k_1)} x(k_1) = e^{N(k_1,k_2) \ln \bar{\mu}}~e^{(k_2 - k_1) \ln \bar{\beta}}~V^{\sigma(k_1)} x(k_1)
\end{align*}
To ensure that SOLA is a contraction, we need that 
\begin{align*}
N(k_1,k_2) \ln u^* + (k_2 - k_1) \ln \bar{\beta} &\leq 0
\end{align*}
To admit at least one switch if $\frac{(k_2 - k_1)}{\tau} < 1$, we introduce $N_0 > 0$.
Therefore, when the number of switches $N(k_1,k_2)$ satisfies
\begin{align}
N(k_1,k_2) &\leq N_0 + \frac{(k_2 - k_1)}{\tau},
\end{align} 
SOLA is a contracting optimizer that achieves online stability. Further, when $0 \leq \alpha(k) \leq 1$, SOLA uses a convex combination of the update by the local algorithm $\A^{\sigma(k)}$ and the parameter at the previous time instance $x(k-1)$. The convex combination of contracting systems also results in a contracting system as shown in~\cite{WL-JJES:1998}. Hence, the overall SOLA algorithm achieves online stability. 

\label{app:selecting-signal}
\end{appendix}


\bibliography{refs}


\end{document}